\documentclass[]{article}

\usepackage{proceed2e}
\usepackage[disable]{todonotes}

\newcommand{\todok}[2][]{\todo[inline,color=orange!20!white,#1]{#2}}
\newcommand{\todot}[2][]{\todo[inline,color=blue!20!white,#1]{#2}}

\newif\ifsup
\suptrue

\usepackage{latexsym}
\usepackage{amsmath}
\usepackage{amssymb}
\usepackage{mathtools}
\usepackage{enumerate}
\usepackage{accents}
\usepackage{tikz}
\usepackage{pgfplots}
\usepackage[ruled]{algorithm}
\usepackage{algpseudocode}
\usepackage{dsfont}
\usepackage{caption}
\usepackage{hyperref}
\hypersetup{
    bookmarks=true,         
    unicode=false,          
    pdftoolbar=true,        
    pdfmenubar=true,        
    pdffitwindow=false,     
    pdfstartview={FitH},    
    pdftitle={My title},    
    pdfauthor={Author},     
    pdfsubject={Subject},   
    pdfcreator={Creator},   
    pdfproducer={Producer}, 
    pdfkeywords={keyword1} {key2} {key3}, 
    pdfnewwindow=true,      
    colorlinks=true,       
    linkcolor=red,          
    citecolor=blue,        
    filecolor=magenta,      
    urlcolor=cyan           
}
\usepackage{amsthm}
\usepackage{times}
\usepackage{natbib}
\usepackage{nicefrac}
\usepackage{wrapfig}
\usepackage{enumitem}
\usepackage[capitalize]{cleveref}

\newcommand{\defined}{\vcentcolon =}
\newcommand{\rdefined}{=\vcentcolon}
\newcommand{\E}{\mathbb E}
\newcommand{\Var}{\operatorname{Var}}
\newcommand{\calF}{\mathcal F}
\newcommand{\sr}[1]{\stackrel{#1}}
\newcommand{\set}[1]{\left\{#1\right\}}
\newcommand{\ind}[1]{\mathds{1}\!\!\set{#1}}

\newcommand{\argmin}{\operatornamewithlimits{arg\,min}}

\newcommand{\ceil}[1]{\left \lceil {#1} \right\rceil}
\newcommand{\eqn}[1]{\begin{align}#1\end{align}}
\newcommand{\eq}[1]{\begin{align*}#1\end{align*}}

\def\subsubsect#1{\vspace{1ex plus 0.5ex minus 0.5ex}\noindent{\bf\boldmath{#1.}}}

\renewcommand{\P}[1]{\mathbb{P}\left\{#1\right\}}

\newcommand{\consttdelta}{{\delta \over 48 \eta^4 n^{6}}}
\newcommand{\consttdeltak}{{\delta \over 48 \eta^4_k n^{6}}}
\newcommand{\constcone}{27 \log {2 \over \tilde\delta}}
\newcommand{\constconek}{27 \log {2 \over \tilde\delta_k}}
\newcommand{\constctwo}{6 \log {2 \over \tilde\delta}}
\newcommand{\constctwok}{6 \log {2 \over \tilde\delta_k}}
\newcommand{\KL}{\operatorname{KL}}

\newcommand{\ubar}[1]{\underline{#1\mkern-3mu}\mkern3mu }
\newcommand{\ber}[1]{\operatorname{Bernoulli}\left(#1\right)}

\let\epsilon\varepsilon

\theoremstyle{plain}
\newtheorem{theorem}{Theorem}
\newtheorem{proposition}[theorem]{Proposition}
\newtheorem{lemma}[theorem]{Lemma}

\theoremstyle{definition}

\newtheorem{remark}[theorem]{Remark}

\theoremstyle{remark}

\ifsup
\else
\allowdisplaybreaks 
\fi

\title{Optimal Resource Allocation with Semi-Bandit Feedback}

\author{ {\bf Tor Lattimore
} \\
Dept.\ of Computing Science \\
University of Alberta, Canada
\And
{\bf Koby Crammer}  \\
Dept.\ of Electrical Engineering \\
The Technion, Israel
\And
{\bf Csaba Szepesv\'ari\thanks{\hspace{0.1cm} On sabbatical leave from the Department of Computing Science, University of Alberta, Canada}}
\\
Microsoft Research\\
Redmond, USA
}

\begin{document}
\maketitle

\begin{abstract}
We study a sequential resource allocation problem involving a fixed number of recurring jobs.
At each time-step the manager should distribute available resources among the jobs in order to maximise the expected number of completed jobs. 
Allocating more resources to a given job increases the probability
that it completes, but with a cut-off. 
Specifically, we assume a linear
model where the probability increases linearly until it equals one, after which
allocating additional resources is wasteful.
We assume the difficulty of each job is unknown
and present the first algorithm for this problem and prove upper and lower bounds on its regret.
Despite its apparent simplicity, the problem has a rich structure: 
we show that an appropriate optimistic algorithm can improve its learning speed dramatically 
beyond the results one normally expects for similar problems
as the problem becomes resource-laden.
\end{abstract}

\section{INTRODUCTION}

Assume that there are $K$ jobs and at each time-step $t$ a learner must distribute the available resources  
with $M_{k,t} \geq 0$ going to job $k$, subject to a budget constraint,
\eq{
\sum_{k=1}^K M_{k,t} \leq 1.
}
The probability that the $k$th job completes in time-step $t$ is
$\min\set{1, M_{k,t} / \nu_k}$, where the unknown cut-off parameter
$\nu_k \in (0, \infty]$ determines the difficulty of job $k$.
After every time-step the resources are replenished
and all jobs are restarted regardless of whether or not they completed successfully in the previous time-step.
The goal of the learner is to maximise the expected number of jobs
that successfully complete up to some known time horizon $n$. 

Despite the simple model, the problem is surprisingly rich.  Given its
information structure, the problem belongs to the class of stochastic
partial monitoring problems, which was first studied by
\citet{AgTeAn89:pmon}%
\footnote{The name was invented later by (perhaps)
  \citep{Rustichini99}.},  where in each time step the learner receives
noisy information about a hidden ``parameter'' while trying to
maximise the sum of rewards and both the information received and the
rewards depend in a known fashion on the actions and the hidden
parameter.  While partial monitoring by now is relatively well
understood, either in the stochastic or the adversarial framework when
the action set is finite \citep{BaPaSze11,FosterR12,Bartok13}, the
case of continuous action sets has received only limited attention \citep[and references therein]{BR12}.
To illustrate the difficulty of the problem, notice that
over-assigning resources to a given job means that the job completes
with certainty and provides little information about the job's
difficulty.  On the other hand, if resources are under-assigned, then
the information received allows one to learn about the payoff
associated with all possible arms, which is reminiscent of bandit
problems where the arms have ``correlated payoffs'' (e.g.,
\citealt{FiOlGaSze10,RuVR13} and the references therein). Finally,
allocating less resources yields high-variance estimates.

        Our motivation to study this particular framework comes from
        the problem of cache allocation.
        \todok{The paper is under review, so we can't reference it} 
        In
        particular, data collected offline from existing and
        experimental allocation strategies showed a relatively good
        fit to the above parametric
        model.  
        In this problem each job is a computer process, which is
        successful in a given time-step if there were no cache misses
        (cache misses are very expensive).  Besides this specific
        resource allocation problem, we also envision other
        applications, such as load balancing in networked
        environments, or any other computing applications where some
        precious resource (bandwidth, radio spectrum, CPU, etc.)  is
        to be subdivided amongst competing processes.  In fact, we
        anticipate numerous extensions and adaptations for specific
        applications, such as in the case of bandits (see,
        \citet{BC12} for an overview of this rich literature).
        Finally, let us point out that although our problem is
        superficially similar to the so-called budgeted bandit
        problems (or, budget limited bandit problems), there are some
        major differences: in budgeted bandits, the information
        structure is still that of bandit problems and the resources
        are not replenished. Either learning stops when the budget is
        exhausted (e.g., \citealt{TTChRoJe12,DiQiZhLi13,BeKeSl13})%
        \footnote{Besides \citet{BeKeSl13}, all works consider finite
          action spaces and unstructured reward functions.  },  or
        performance is measured against the total resources consumed
        in an ongoing fashion (e.g., \citealt{GyKoSzSz07}).

The main contribution besides the introduction of a new problem is a new optimistic algorithm  for this problem that is shown to suffer poly-logarithmic regret with respect to optimal omniscient algorithm that knows the parameters $(\nu_k)_k$ in advance. 
The structure of the bound depends significantly on the problem dynamics, ranging from a (relatively) easy full-information-like setting, corresponding to a resource-laden regime, to a bandit-like setting, corresponding to the resource-scant setting. 
Again, to contrast this work to previous works, note that the results we obtain for the full-information-like setting are distinct from those possible in the finite action case, where the full-information setting allows one to learn with finite regret
\citep{AgTeAn89:pmon}.
On the technical side, we believe that our study and use of weighted estimators in situations where some samples are more informative than others might be of independent interest, too.

Problems of allocating resources to jobs were studied in the community
of architecture and operating systems. \citet{liu2004organizing} build
static profile-based allocation of L2-cache banks to different
processes using their current miss rate data. 
%
\citet{suh2002new} proposed a hit-rate optimisation using hardware
counters which used a model-based estimation of hit-rate vs allocated
cache.  However, they all assume the model is fully known and no
learning is required. \citet{bitirgen2008coordinated} used ANNs to
predict individual program performance as a function of resources.
Finally, \citet{Ipek:2008:SMC:1394608.1382172} used reinforcement
learning to allocate DRAM to multi-processors. 


\section{PRELIMINARIES}

In each time-step $t$ the learner chooses $M_{k,t} \geq 0$ subject to the constraint, $\sum_{k=1}^K M_{k,t} \leq 1$. Then all jobs are executed
and $X_{k,t} \in \set{0,1}$
indicates the success or failure of job $k$ in time-step $t$ and 
is sampled from a Bernoulli distribution with parameter $\beta(M_{k,t} / \nu_k) \defined \min\set{1, M_{k,t} / \nu_k}$.
The goal is to maximise the expected number of jobs that successfully complete, 
$\scriptstyle\sum_{k=1}^K \beta(M_{k,t} / \nu_k)$.
We define the gaps $\Delta_{j,k} = \nu_j^{-1} - \nu_k^{-1}$.
We assume throughout for convenience, and without loss of generality, that $\nu_1 < \nu_2 < \cdots < \nu_K$. It can be shown that 
the optimal allocation distributes the resources to jobs in increasing order of difficulty.
\eq{
M^*_k = \min \set{1 - \sum_{i=1}^{k-1} M^*_i,\; \nu_k}.
}
\todot{This seems a bit straight-forward for a theorem, but still maybe not so easy for the reader}
We let $\ell$ be the number of jobs that are fully allocated under the optimal policy:
$\ell = \max \set{i : M^*_i = \nu_i}$.
The overflow is denoted by $S^* = M^*_{\ell+1}$, which we assume to vanish if $\ell = K$.
The expected reward (number of completed jobs) when following the optimal allocation is
\eq{
\sum_{k=1}^K {M^*_k \over \nu_k} = \ell + {S^* \over \nu_{\ell+1}},
}
where we define $\nu_{K+1} = \infty$ in the case that $\ell = K$.
The (expected $n$-step cumulative) regret of a given allocation algorithm is the difference between the expected number of jobs that 
complete under the optimal policy and those that complete given the algorithm,
\eq{
R_n &= \E\left[\sum_{t=1}^n r_t\right], 
&r_t
& = \sum_{k=1}^K \beta(M_k^* / \nu_k) - \sum_{k=1}^K \beta(M_{k,t}/\nu_k) \\ 
&&&= \left(\ell + {S^* \over \nu_{\ell+1}} \right) - \sum_{k=1}^K \beta(M_{k,t} / \nu_k).
}

\section{OVERVIEW OF ALGORITHM}

We take inspiration from the optimal policy for known $\nu_k$, 
	which is to fully allocate the jobs with the smallest $\nu_k$ (easiest jobs) and allocate
the remainder/overflow to the next easiest job. 
At each time-step $t$ we replace the unknown $\nu_k$ by a high-probability lower bound $\ubar \nu_{k,t-1} \leq \nu_k$. 
This corresponds to the optimistic strategy, which assumes that each job is as easy as reasonably possible.
The construction of a confidence interval about $\nu_k$ is surprisingly delicate. 
There are two main challenges. First, the function $\beta(M_{k,t} / \nu_k)$ is non-differentiable at $M_{k,t} = \nu_k$, and for $M_{k,t} \geq \nu_k$ the 
job will always complete and little information is gained.
This is addressed by always using a lower estimate of $\nu_k$ in the algorithm.
The second challenge is that $M_{k,t}$ will vary with time, so the samples $X_{k,t}$ are not identically distributed. This would normally be unproblematic,
since martingale inequalities can be applied, but the specific structure of this problem
means that a standard sample average estimator is a little weak in the sense
that its estimation accuracy can be dramatically improved.
In particular, we will propose an estimator that is able to take advantage of the fact that
the variance of $X_{k,t}$ decreases to zero as $M_{k,t}$ approaches $\nu_k$ from below.

As far as the estimates are concerned,
rather than estimate the parameters $\nu_k$, it turns out that learning the reciprocal $\nu^{-1}_k$ is both more approachable and ultimately more useful
for proving regret bounds. Fix $k$
and let $M_{k,1}, \ldots, M_{k,t} \leq \nu_k$ be a sequence of allocations with $M_{k,s} \leq \nu_k$ and 
$X_{k,s} \sim \ber{M_{k,s} / \nu_k}$. Then a natural (unbiased) estimator of $\nu^{-1}_k$ is given by
\eq{
{1 \over \hat \nu_{k,t}} \defined {1 \over t} \sum_{s=1}^t {X_{k,s} \over M_{k,s}}.
}
The estimator has some interesting properties. 
First, the random variable $X_{k,s} / M_{k,s} \in [0, 1/M_{k,s}]$ has a large range for small $M_{k,s}$, which
makes it difficult to control the error $\hat \nu^{-1}_{k,t} - \nu^{-1}_k$ via the usual Azuma/Bernstein inequalities.
Secondly, if $M_{k,s}$ is close to $\nu_k$, then the range of $X_{k,s} / M_{k,s}$ is small, which makes estimation easier. Additionally, the variance is greatly decreased
for $M_{k,s}$ close to $\nu_k$. 
This suggests that samples for which $M_{k,s}$ is large are more useful than those where $M_{k,s}$ is small, which motivates the use of the weighted estimator,
\eq{
{1 \over \hat \nu_{k,t}} \defined {\sum_{s=1}^t {w_s X_{k,s}} \over \sum_{s=1}^t w_s M_{k,s}},
}
where $w_s$ will be chosen in a data-dependent way, but with the important characteristic that $w_s$ is large for $M_{k,s}$ close to $\nu_k$.
The pseudo-code of the main algorithm is shown on Algorithm Listing \ref{alg:OAA}. It accepts as input the horizon $n$, the number of jobs, and 
a set $\set{\ubar \nu_{k,0}}_{k=1}^K$ for which $0 < \ubar \nu_{k,0} \leq \nu_k$ for each $k$. In \cref{sec:init} we present a simple 
(and efficient) algorithm that relaxes the need for the lower bounds $\ubar \nu_{k,0}$.

\begin{algorithm}[t]
\caption{Optimistic Allocation Algorithm}\label{alg:OAA}
\begin{algorithmic}[1]
\State {\bf input: } $n, K$, $\set{\ubar \nu_{k,0}}_{k=1}^K$
\State $\delta \leftarrow (nK)^{-2}$ and $\bar \nu_{k,0} = \infty$ for each $k$
\For{$t \in 1, \ldots, n$}
\State /* Optimistically choose $M_{k,t}$ using $\ubar \nu_{k,t-1}$ */
\State $(\forall k \in 1,\ldots,K)$ initialise $M_{k,t} \leftarrow 0$
\For{$i \in 1, \ldots, K$}
\State $k \leftarrow \argmin\limits_{k : M_{k,t} = 0} \ubar \nu_{k,t-1}$
\State $M_{k,t} \leftarrow \min \set{\ubar \nu_{k,t-1},\; 1 - \sum_{j=1}^K M_{j,t}}$
\EndFor
\State $(\forall k \in 1,\ldots,K)$ observe $X_{k,t}$ 
\State $(\forall k \in 1,\ldots,K)$ compute weighted estimates:
\eq{
w_{k,t} &\leftarrow {1 \over 1 - {M_{k,t} \over \bar \nu_{k,t-1}}} 
&{1 \over \hat\nu_{k,t}} &\leftarrow {\sum_{s=1}^{t} w_{k,s} X_{k,s} \over \sum_{s=1}^t w_{k,s} M_{k,s}} 
}
\State $(\forall k \in 1,\ldots,K)$ update confidence intervals:
\eq{
R_{k,t} &\leftarrow \max_{s \leq t} w_{k,s}  \qquad \hat V^2_{k,t} \leftarrow \sum_{s \leq t} {w_{k,s} M_{k,s} \over \ubar \nu_{k,t-1}} \\
\tilde \epsilon_{k,t} &\leftarrow {f(R_{k,t}, \hat V^2_{k,t},\delta) \over \sum_{s=1}^t w_{k,s} M_{k,s}}    \\
{1 \over \ubar \nu_{k,t}} &\leftarrow \min\set{{1 \over \ubar \nu_{k,t-1}}, {1 \over \hat \nu_{k,t}} + \tilde\epsilon_{k,t}} \\
{1 \over \bar \nu_{k,t}} &\leftarrow \max\set{{1 \over \bar \nu_{k,t-1}}, {1 \over \hat \nu_{k,t}} - \tilde\epsilon_{k,t}} 
}
\EndFor
\Statex 
\Function{$f$}{$R$,$V^2$, $\delta$}
\State $\delta_0 \leftarrow {\delta \over 3(R+1)^2 (V^2+1)^2}$
\State \Return 
${R+1 \over 3} \log{2 \over \delta_0}$
\Statex $\qquad +\;\;\sqrt{2(V^2+1) \log{2 \over \delta_0} + \left({R+1 \over 3}\right)^2 \log^2{2 \over \delta_0}}$
\EndFunction
\end{algorithmic}
\end{algorithm}

\begin{remark}
Later (in Lemma \ref{lem:w}) we will show that with high probability $1 \leq w_{k,s} \leq O(s)$. By definition the confidence bounds
$\ubar \nu_{k,t}$ and $\bar \nu_{k,t}$ are non-decreasing/increasing respectively. These results are sufficient to guarantee that the 
new algorithm is numerically stable. It is also worth noting that the running time of Algorithm \ref{alg:OAA} is $O(1)$ per time step, since
all sums can be computed incrementally.
\end{remark}

\section{UPPER BOUNDS ON THE REGRET}\label{sec:theorems}

The regret of \cref{alg:OAA} depends in a subtle way on the parameters $\nu_k$. 
There are four natural cases, which will appear in our main result.

\subsubsect{Case 1: Insufficient budget for any jobs}
In this case $\ell = 0$ and the optimal algorithm allocates all available resources to the easiest task, which means $M^*_1 = 1$.
Knowing that $\ell=0$, the problem can be reduced to a $K$-armed Bernoulli bandit by restricting the action space to $M_{k,t} = 1$ for all $k$.
Then a bandit algorithm such as UCB1 \citep{ACF02} will achieve logarithmic (problem dependent) 
	regret with some dependence on the gaps $\Delta_{1,k} = {1 \over \nu_1} - {1 \over \nu_k}$.
In particular, the regret looks like 
$R_n \in O\left(\sum_{k=2}^K {\log n \over \Delta_{1,k}}\right)$.

\subsubsect{Case 2: Sufficient budget for all jobs}
In this case $\ell = K$ and the optimal policy assigns $M_{k,t} = \nu_k$ for all $k$, which enjoys a reward of $K$ at each time-step.
Now \cref{alg:OAA} will choose $M_{k,t} = \ubar \nu_{k,t-1}$ for all time-steps and by \cref{thm:fast-estimator} stated below
we will have  $\ubar \nu_{k,t-1} / \nu_k \in O(1 - {1 \over t} \log n)$. Consequently, the regret may be bounded by
$R_n \in O\left(\log^2 n\right)$
with \emph{no dependence on the gaps}.

\subsubsect{Case 3: Sufficient budget for all but one job}
Now the algorithm must learn which jobs should be fully allocated. This introduces a weak dependence on the gaps $\Delta_{\ell,k}$ for $k > \ell$,
but choosing the overflow job is trivial. Again we expect the regret to be $O(\log^2 n)$, but with an additional modest dependence on the gaps.

\subsubsect{Case 4: General case}
In the completely general case even the choice of the overflow job is non-trivial. Ultimately it turns out that in this setting the problem
decomposes into two sub-problems. Choosing the jobs to fully allocate, and choosing the overflow job. The first component is fast, since
we can make use of the faster learning when fully allocating. Choosing the overflow reduces to the bandit problem as described in case 1.

Our main result is the following theorem bounding the regret of our algorithm.
\begin{theorem}\label{thm:main}
Let $\delta$ be as in the algorithm,
$\displaystyle 
\eta_k = {\min\set{1, \nu_k} / \ubar \nu_{k,0}}$,
$\tilde\delta_k = \consttdeltak$, 
$c_{k,1} = \constconek$,
$c_{k,2} = \constctwok$,
$u_{k,j} = {c_{k,1} \over \ubar \nu_{k,0} \Delta_{j,k}}$. 
Then \cref{alg:OAA} suffers regret at most
\eq{
&R_n \leq
1 
+ \sum_{k=1}^\ell c_{k,1} \eta_k(1 + \log n) \\ 
&+ \ind{\ell < K}\! \Bigg[
\sum_{k=\ell+2}^K {c_{k,2} \over \ubar \nu_{k,0} \Delta_{\ell+1,k}} 
+ \sum_{k=1}^{\ell+1} c_{k,1}\eta_k(1 + \log n) \\ 
& + \sum_{\mathclap{k=\ell+2}}^K c_{k,1} \eta_k(1 + \log u_{\ell+1,k}) 
+ \sum_{\mathclap{k = \ell+1}}^K c_{k,1} \eta_k(1 + \log u_{\ell,k}) 
\Bigg].
}
\end{theorem}

If we assume $\eta_k \in O(1)$ for each $k$ (reasonable as discussed in \cref{sec:init}), then the regret bound looks like
\eqn{
\label{eq:big-o-theorem}
&R_n \in O\Bigg(\ell \log^2 n + \sum_{k=\ell+1}^K \left(\log {1 \over \nu_k \Delta_{\ell,k}}\right)\log n  \\
\nonumber &\;\;+\sum_{k=\ell+2}^K \left(\log{1 \over \nu_k \Delta_{\ell+1,k}}\right) \log n
+ \sum_{k=\ell+1}^K {\log n \over \Delta_{\ell+1,k}} \Bigg),
}
where the first term is due to the gap between $\ubar \nu_{k,t}$ and $\nu_k$, the second due to discovering which jobs should be fully allocated, 
while the third and fourth terms are due to mistakes when choosing the overflow job.

The proof is broken into two components. In the first part we tackle the convergence of $\hat \nu_{t,k}$ to $\nu_k$ and analyse the width
of the confidence intervals, which are data-dependent and shrink substantially faster when $M_{k,t}$ is chosen
close to $\nu_k$. In the second component we decompose the regret in terms of the width of the confidence intervals. 
While we avoided large constants in the algorithm itself, in the proof we focus on legibility. Optimising the
constants would complicate an already long result.

\section{ESTIMATION}\label{sec:estimation}

We consider a single job with parameter $\nu$ and analyse the estimator and confidence intervals used
by \cref{alg:OAA}. We start by showing that the confidence intervals contain the truth with high-probability and then
analyse the rate at which the intervals shrink as more more data is observed. Somewhat surprisingly the rate has
a strong dependence on the data with larger allocations leading to faster convergence.

Let $\set{\calF_t}_{t=0}^\infty$ be a filtration and let $M_1, \ldots, M_n$ be a sequence of positive random variables such that $M_t$ is $\calF_{t-1}$-measurable.
Define $X_t$ to be sampled from a Bernoulli distribution with parameter $\beta(M_t / \nu)$ for some $\nu \in [\ubar \nu_0, \infty]$ and
assume that $X_t$ is $\calF_t$-measurable.
Our goal is to construct a sequence of confidence intervals $\set{[\ubar \nu_t, \bar \nu_t]}_{t=1}^n$ such that $\nu \in [\ubar \nu_t, \bar \nu_t]$ with
high probability and $\bar \nu_t - \ubar \nu_t \to 0$ as fast as possible. We assume a known lower bound $\ubar \nu_0 \leq \nu$ and define $\bar \nu_0 = \infty$.
Recall that the estimator used by \cref{alg:OAA} is defined by
\eq{
w_s & = {1 \over 1 - {M_t \over \bar \nu_{t-1}}}\,, &
{1 \over \hat \nu_t} &= {\sum_{s=1}^t w_s X_s \over \sum_{s=1}^t w_s M_s}\,.
}
Fix a number $0<\delta<1$ and define
$\tilde \epsilon_t = f(R_t, \hat V_t^2,\delta) / \sum_{s=1}^t w_s M_s$,
where the function $f$ is defined in \cref{alg:OAA}, $R_t = \max_{s \leq t} w_s$ and $\hat V_t^2 = \sum_{s=1}^t {w_s M_s \over \ubar \nu_{t-1}}$.
The lower and upper confidence bounds on $\nu^{-1}$ are defined by,
\eq{
{1 \over \ubar \nu_t}\! =\! \min\set{{1 \over \ubar \nu_{t-1}},{1 \over \hat \nu_t} + \tilde \epsilon_t},\;
{1 \over \bar \nu_t}\! =\! \max\set{{1 \over \bar \nu_{t-1}}, {1 \over \hat \nu_t} - \tilde \epsilon_t}.
}
We define $\epsilon_t = \ubar \nu_{t}^{-1} - \bar \nu_{t}^{-1}$ to be the (decreasing) width of the
confidence interval.
Note that both $\ubar \nu_t$ and $\bar \nu_t$ depend on $\delta$, although this dependence is not shown to minimise clutter.

\begin{theorem}\label{thm:valid-confidence}
If $M_s$ is chosen such that $M_s \leq \ubar \nu_{s-1}$ for all $s$ 
then $\P{ \exists s \leq t \text{ s.t. } \nu \not\in  [\ubar \nu_s, \bar \nu_s] } \le t \delta$
holds for any $0<\delta<1$.
\end{theorem}
\begin{proof}[Proof of \cref{thm:valid-confidence}]
Let $F_t$ be the event
$F_t =\left\{ \nu \in [\ubar \nu_t, \bar \nu_t] \right\}$.
Note that since $ [\ubar \nu_t, \bar \nu_t] \subset  [\ubar \nu_{t-1}, \bar \nu_{t-1}] \subset \cdots \subset
 [\ubar \nu_0, \bar \nu_0]$, $F_t \subset F_{t-1} \subset \cdots \subset F_0$.
Hence, $F_t = \cap_{s\le t} F_s$
and
 it suffices to prove that $\P{ F_t^c} \leq t\delta$.%
\footnote{For an event $E$, we use $E^c$ to denote its complement.}

Define $Y_s = w_s X_s - {w_s M_s \over \nu}$ and $S_t = \sum_{s=1}^t Y_s$
and $V_t^2 = \sum_{s=1}^t \Var[Y_s|\calF_{s-1}]$. 
We proceed by induction. Assume $\P{ F_{t-1}^c} \leq (t - 1)\delta$, which is trivial for $t = 1$.
Now, on $F_{t-1}$,
\eq{
&V_t^2 
\sr{(a)}= \sum_{s=1}^t \Var[Y_s|\calF_{s-1}] 
\sr{(b)}= \sum_{s=1}^t {w_s^2 M_s \over \nu}\left(1 - {M_s \over \nu}\right) \\
&\sr{(c)}= \sum_{s=1}^t {w_s M_s \over \nu} \left({1 - {M_s \over \nu} \over 1 - {M_s \over \bar \nu_{s-1}}}\right) 
\sr{(d)}\leq \sum_{s=1}^t {w_s M_s \over \nu} 
\sr{(e)}\leq \hat V_t^2,
}
where (a) is the definition of $V_t^2$, (b) follows since $w_s$ is $\calF_{s-1}$-measurable,
(c) follows by substituting the definition of $w_s$,
(d) and (e) are true since given $F_{t-1}$ we know that $\ubar \nu_{s-1} \leq \nu \leq \bar \nu_{s-1}$.
Therefore $f(R_t, V_t^2,\delta) \leq f(R_t, \hat V_t^2,\delta)$, which follows since $f$ is monotone increasing in its second argument. Therefore, 
\eqn{
\nonumber &\P{\left|{1 \over \hat \nu_t} - {1 \over \nu}\right| \geq \tilde \epsilon_t \wedge F_{t-1}} \\
\nonumber =~& \P{\left|{\sum_{s=1}^t w_s X_s \over \sum_{s=1}^t w_s M_s} - {1 \over \nu}\right| \geq {f(R_t, \hat V_t^2,\delta) \over \sum_{s=1}^t w_s M_s} \wedge F_{t-1}} \\
\nonumber \leq~& \P{\left|\sum_{s=1}^t w_s X_s - \sum_{s=1}^t {w_s M_s \over \nu}\right| \geq f(R_t, V_t^2,\delta) \wedge F_{t-1}} \\
\label{eq:introduce_Y} =~& \P{|S_t| \geq f(R_t, V_t^2,\delta) \wedge F_{t-1}}.
}
By the union bound we have
\eq{
&\P{|S_t| \geq f(R_t, \hat V_t^2,\delta) \vee  F_{t-1}^c} \\
\leq~& \P{|S_t| \geq f(R_t, V_t^2,\delta) \wedge F_{t-1}} + \P{ F_{t-1}^c} \\
\sr{(a)}\leq~& \delta + \P{ F_{t-1}^c} 
\leq \delta + (t - 1)\delta 
= t\delta\,,
}
\ifsup
where (a) follows from \cref{thm:peeled-bernstein} in the Appendix.
\fi
Therefore $\P{|S_t| \leq f(R_t, V_t^2,\delta) \wedge F_{t-1}} \geq 1 -
t\delta$ and so with probability at least $1 - t\delta$ we have that $F_{t-1}$ and
\eq{
\left|{1 \over \hat \nu_t} - {1 \over \nu}\right| \leq {f(R_t, \hat V_t^2,\delta) \over \sum_{s=1}^t w_s M_s} = \tilde\epsilon_t,
}
in which case
\eq{
{1 \over \ubar \nu_t} = \min\set{{1 \over \ubar \nu_{t-1}}, {1 \over \hat \nu_t} + \tilde \epsilon_t} \geq {1 \over \nu}~,
}
and similarly ${1 \over \bar \nu_t} \leq {1 \over \nu}$, which implies $F_t$. 
Therefore $\P{ F_t^c} \leq t\delta$ as required. 
\end{proof}


We now analyse
the width $\epsilon_t \equiv \ubar \nu_t^{-1} - \bar \nu_t^{-1}$ of the confidence interval 
obtained after $t$ samples are observed.
We say that a job is {\it fully allocated} at time-step $s$ if $M_s = \ubar \nu_{s-1}$.
The first theorem shows that the width $\epsilon_t$ drops with order
$O({1 / T(t)})$,
where $T(t) = \sum_{s=1}^t \ind{M_s = \ubar \nu_{s-1}}$ is the number of fully allocated time-steps.
The second theorem shows that for any $\alpha > 0$, the width $\epsilon_t$ drops with order
$O(\sqrt{1 /(\alpha U_\alpha(t))})$,
where $U_\alpha(t) = \sum_{s=1}^t \ind{M_s \geq \alpha}$.
The dramatic difference in speeds is due to the low variance $\Var[X_t|\calF_{t-1}]$ when $M_t$ is chosen close to $\nu$. 
For the next results define $\eta = \min \set{1, \nu} / \ubar \nu_0$ and 
$\tilde\delta = \consttdelta$.

\begin{theorem}\label{thm:fast-estimator}
$\displaystyle \epsilon_t \leq {c_1 \over \ubar \nu_0 (T(t) + 1)}$ where $c_1 = \constcone$.
\end{theorem}

\begin{theorem}\label{thm:slow-estimator}
$\displaystyle\epsilon_t \leq \sqrt{c_2 \over \alpha \ubar \nu_0 U_\alpha(t)}$ where $c_2 = \constctwo$.
\end{theorem}
The proofs are based on the following lemma that collects some simple observations:
\begin{lemma}\label{lem:w}
The following hold for any $t\ge 1$:
\begin{enumerate}[noitemsep,topsep=0pt,parsep=0pt,partopsep=0pt]
\item \label{lem:w1} $w_t M_t \leq {1 \over \epsilon_{t-1}}$, with equality if $M_t = \ubar \nu_{t-1}$.
\item \label{lem:w2} $1 \leq R_t \leq {1 \over \ubar \nu_0 \epsilon_{t-1}}$.
\item \label{lem:w3} $\epsilon_t \geq {1 \over t\min\set{1,\nu}}$.
\item \label{lem:w4} $1 - {\ubar \nu_t \over \nu} \leq \ubar \nu_t \epsilon_t$.
\end{enumerate}
\end{lemma}

\begin{proof}
Using the definition of $w_s$ and the fact that $M_s$ is always chosen to be
smaller or equal to $\ubar \nu_{s-1}$, we get
\eq{
w_s &\equiv \left(1 - {M_s \over \bar \nu_{s-1}}\right)^{-1}
\sr{(a)}\leq \left(1 - {\ubar \nu_{s-1} \over \bar \nu_{s-1}}\right)^{-1} 
= {1 \over \epsilon_{s-1} \ubar \nu_{s-1}}\,.
}
The first claim follows since the inequality (a) can be replaced by equality if $M_s = \ubar \nu_{s-1}$.
The second follows from the definition of $R_t$ and the facts that $(\epsilon_s)_s$ is non-increasing and $(\ubar \nu_s)_s$ is non-decreasing.
For the third claim we recall that $R_t = \max_{s \leq t} w_s$ and $M_s \leq \nu$. Therefore,
\eq{
\epsilon_t &\sr{(a)}\geq \min\set{\epsilon_{t-1}, {R_t \over \sum_{s=1}^t w_s M_s}} \\ 
&\sr{(b)}\geq \min\set{\epsilon_{t-1}, {1 \over t\min\set{1,\nu}}},
}
where (a) follows from the definition of $\epsilon_t$ and naive bounding of the function $f$, (b) 
follows since $R_t \geq w_s$ for all $s \leq t$ and because $M_s \leq \min\set{1, \nu}$ for all $s$.
Trivial induction and the fact that $\epsilon_0 = \ubar \nu_0^{-1} \geq \nu^{-1}$ completes the proof of the third claim.
For the final claim we use the facts that $\ubar \nu_t^{-1} \leq
\nu^{-1} + \epsilon_t$. Therefore, 
$1 - {\ubar \nu_t \over \nu_t} = \ubar \nu_t \left({1 \over \ubar \nu_t} - {1 \over \nu}\right)
\leq \ubar \nu_t \epsilon_t$.
\end{proof}

\begin{lemma}\label{lem:eps}
$\displaystyle \epsilon_t \leq
{6R_t \log{2 \over \tilde \delta} \over \sum_{s=1}^t w_s M_s} + \sqrt{{8 \log{2 \over \tilde \delta} \over \ubar \nu_0 \sum_{s=1}^t w_s M_s }}$.
\end{lemma}

\begin{proof}
Let $\delta_t = \delta / (3(R_t+1)^2(\hat V_t^2 + 1)^2) < 1$. 
By the definition of $\epsilon_t$,
\eq{
\epsilon_t& 
\leq {2f(R_t, \hat V_t^2,\delta) \over \sum_{s=1}^t w_s M_s} \\
&\sr{(a)}\leq {{4 (R_t+1) \over 3} \log{2 \over \delta_t} + 2\sqrt{2(\hat V_t^2 + 1) \log{2 \over \delta_t}} \over \sum_{s=1}^t w_s M_s} \\
&\sr{(b)}\leq {6R_t \log{2 \over \delta_{t}} + \sqrt{{8 \over \ubar \nu_{0}} \sum_{s=1}^t w_s M_s \log{2 \over \delta_{t}}}  \over \sum_{s=1}^t w_s M_s} \\
&= {{6R_t} \log{2 \over \delta_{t}} \over \sum_{s=1}^t w_s M_s} + \sqrt{{8 \log{2 \over \delta_t} \over \ubar \nu_{0}\sum_{s=1}^t w_s M_s }} ,
}
where in (a) we used the definition of $f$, 
in (b) we substituted the definition of $\hat V^2_t$ and 
used the facts that $R_t \geq 1$ and $\ubar \nu_0 \leq \ubar \nu_{t-1}$ and we also used a naive bound.
The proof is completed by proving ${2 / \delta_t} \leq {2 / \tilde \delta}$. Indeed, by \cref{lem:w}, $1 \leq R_t \leq {1 \over \epsilon_{t-1} \ubar \nu_0} \leq {1 \over \epsilon_t \ubar \nu_0}$. We also have $\hat V_t^2 \leq t R_t^2$. Thus,
\eq{
{2 \over \delta_t}
&= {6(R_t+1)^2(\hat V_t^2+1)^2 \over \delta}
\leq {6 \over \delta}\left({16t^2 \over \left(\epsilon_t \ubar \nu_0\right)^4}\right) 
\sr{(a)}\leq {2 \over \tilde\delta}\,,
}
where in (a) we used \cref{lem:w}\eqref{lem:w3}.
\end{proof}

\begin{proof}[Proof of \cref{thm:fast-estimator}]
By \cref{lem:eps},
\eqn{
\label{eq:eps}
\epsilon_t \leq
{6R_t \log{2 \over \tilde \delta} \over \sum_{s=1}^t w_s M_s} + \sqrt{{8 \over \ubar \nu_0 \sum_{s=1}^t w_s M_s } \log{2 \over \tilde\delta}}\,.
}
We proceed by induction. Assume that $\epsilon_{s-1} \leq {c_1 \over \ubar \nu_0 (T(s-1)+1)}$, which is trivial for $s = 1$.
By \cref{lem:w}\eqref{lem:w1},
\eqn{
\label{eq:epsstar} \sum_{s=1}^t w_s M_s 
&\geq \sum_{s=1}^{T(t)} {s \ubar \nu_0 \over c_1} 
= {\ubar \nu_0 T(t) (T(t) + 1) \over 2c_1}.
}
Therefore,
\eqn{
\label{eq:fast-root-part}
\sqrt{{8 \over \ubar \nu_{0}\sum_{s=1}^t w_s M_s} \log{2 \over \tilde\delta}} 
\sr{(a)}\leq {1 \over \ubar \nu_0 T(t)} \sqrt{4c_1 \log {2 \over \tilde\delta}}\,.
}
Now we work on the first term in \eqref{eq:eps}. If $\epsilon_{t-1} \leq {c_1 \over \ubar \nu_0 (T(t)+1)}$, then we are done, since $\epsilon_s$ is non-increasing.
Otherwise, we use \cref{lem:w}\eqref{lem:w2} to obtain,
\eqn{
\nonumber {6 R_t \over \sum_{s=1}^t w_s M_s} \log{2 \over \tilde\delta} 
&\leq {6 \over  \ubar\nu_0 \epsilon_{t-1} \sum_{s=1}^t w_s M_s} \log{2 \over \tilde\delta}\\
\label{eq:fast-remainder-part}&\sr{(a)}\leq {3 \over \ubar \nu_0 T(t)} \log{2 \over \tilde \delta},
}
where in (a) we used \eqref{eq:epsstar} and the lower bound on $\epsilon_{t-1}$.
Substituting \eqref{eq:fast-root-part} and \eqref{eq:fast-remainder-part} into \eqref{eq:eps} we have
\eq{
\epsilon_t \leq {1 \over \ubar \nu_0 T(t)} \sqrt{4c_1 \log {2 \over \tilde\delta}} + {3 \over \ubar \nu_0 T(t)} \log{2 \over \tilde\delta}.
}
Choosing $c_1 = \constcone$ leads to
\eq{
\epsilon_t &\leq {1 \over \ubar \nu_0 T(t)} \sqrt{4\cdot 27  \log^2 {2 \over \tilde\delta}} + {3 \over \ubar\nu_0 T(t)} \log{2 \over \tilde\delta} \\
&\leq {27 \over \ubar \nu_0(T(t)+1)} \log{2 \over \tilde\delta} = {c_1 \over \ubar \nu_0(T(t)+1)},
}
which completes the induction and proof.
\end{proof}

\begin{proof}[Proof of \cref{thm:slow-estimator}]
Firstly, by \cref{lem:eps},
\eq{
\epsilon_t 
&\leq {6R_t \over \sum_{s=1}^t w_s M_s} \log{2 \over \tilde\delta} + \sqrt{{8 \over \ubar \nu_{0} \sum_{s=1}^t w_s M_s} \log{2 \over \tilde\delta}}.
}
The second term is easily bounded by using the fact that $w_s \geq 1$ and the definition of $U_\alpha(t)$,
\eq{
\sqrt{{8 \over \ubar \nu_{0} \sum_{s=1}^t w_s M_s} \log{2 \over \tilde\delta}}
&\leq \sqrt{{8 \over \ubar \nu_0 U_\alpha(t)\alpha} \log{2 \over \tilde\delta}}.
}
For the first term we assume $\epsilon_{t-1} \geq {\sqrt{c_2 \over \ubar \nu_0 U_\alpha(t)\alpha}}$, since otherwise we can apply monotonicity of $\epsilon_t$.
Therefore
\eq{
&{6R_t \over \sum_{s=1}^t w_s M_s} \log{2 \over \tilde\delta}
\leq {6 \over  \ubar \nu_0 \epsilon_{t-1} \sum_{s=1}^t w_s M_s} \log{2 \over \tilde\delta} \\
\leq~& \sqrt{U_{\alpha}(t) \alpha \ubar \nu_0 \over c_2} \cdot {6 \log {2 \over \tilde\delta} \over \ubar \nu_0 U_\alpha(t) \alpha} 
\leq 6\sqrt{1  \over c_2\alpha \ubar \nu_0U_\alpha(t)} \log{2 \over \tilde\delta}.
}
Now choose $c_2 = \constctwo$ to complete the result.
\end{proof}

\section{PROOF OF THEOREM \ref{thm:main}}\label{sec:main-proof}

We are now ready to use the results of \cref{sec:estimation} to bound the regret of \cref{alg:OAA}. 
The first step is to
decompose the regret into two cases depending on whether or not the confidence intervals contain the truth. The probability
that they do not is low, so this contributes negligibly to the regret. When the confidence intervals are valid we break
the problem into two components. The first is the selection of the processes to fully allocation, which leads to the $O(\log^2 n)$ 
part of the bound. The second component involves analysing the selection of the overflow process, where the approach is reminiscent 
of the analysis for the UCB algorithm for stochastic bandits \citep{ACF02}. 

Let $F_{k,t}$ denote the event when none of the confidence intervals underlying job $k$ fail up to time $t$:
\eq{
F_{k,t} = \left\{\forall s \leq t : \nu \in [\ubar \nu_{k,s}, \bar \nu_{k,s}]\right\}\,.
}
The algorithm uses $\delta = (nK)^{-2}$, which is sufficient by a union bound and \cref{thm:valid-confidence} to ensure that,
\eqn{
\label{eq:failure-bound}
 \P{ G^c} \leq {1 \over nK}\,, \quad \text{ where }
G &= \bigcap_{k=1}^K F_{k,n}\,.
}
The regret can be decomposed into two cases depending on whether $G$ holds:
\eqn{
\label{eq:regret-d1} 
&R_n 
= \E \sum_{t=1}^n r_t 
\sr{(a)}= \E \,\ind{ G^c} \sum_{t=1}^n r_t + \E \,\ind{G} \sum_{t=1}^n r_t \\
\nonumber &\sr{(b)}\leq \E \,\ind{ G^c} nK + \E\, \ind{G} \sum_{t=1}^n r_t  
\sr{(c)}\leq 1 + \E\, \ind{G} \sum_{t=1}^n r_t,
}
where
(a) follows from the definition of expectation,
(b) is true by bounding $r_t \leq K$ for all $t$, and
(c) follows from \eqref{eq:failure-bound}.
For the remainder we assume $G$ holds and use \cref{thm:fast-estimator,thm:slow-estimator} combined
with the definition of the algorithm to control the second term in \eqref{eq:regret-d1}.
The first step is to decompose the regret in round $t$:
\eq{
r_t = \ell^* + {S^* \over \nu_{\ell+1}} - \sum_{k=1}^K \beta\left({M_{k,t} \over \nu_k}\right).
}
By the assumption that $G$ holds we know for all $t \leq n$ and $k$ that
$\bar \nu_{k,t}^{-1} \leq \nu_k^{-1} \leq \ubar \nu_{k,t}^{-1}$.
Therefore $M_{k,t} \leq \ubar \nu_{k,t-1} \leq \nu_k$, which means that $\beta(M_{k,t} / \nu_k) = M_{k,t} / \nu_k$.
Define $\pi_t(i) \in \set{1, \ldots, K}$ such that $\ubar\nu_{\pi_t(i),t-1} \leq \ubar\nu_{\pi_t(i+1),t-1}$.
Also let
\eq{
A_t &= \set{k : M_{k,t} = \ubar \nu_{k,t-1}}, \\ 
A_t^{\leq j} &= A_t \cap \set{\pi_i(t) : 1 \leq i \leq j},\\
T_k(t) &= \sum_{s=1}^t \ind{k \in A_t} \quad \text{ and}  \quad
B_t = \pi_t(\ell+1). 
}
Informally, $A_t$ is the set of jobs that are fully allocated at time-step $t$, 
$A_t^{\leq j}$ is a subset of $A_t$ containing the $j$ jobs believed to be easiest,
$T_k(t)$ is the number of times job $k$ has been fully allocated at time-step $t$, and
$B_t$ is the $(\ell+1)$th easiest job at time-step $t$ (this is only defined if $\ell < K$ and will only be used in that case).

\begin{lemma}\label{lem:s}
For all $t$, $|A_t| \geq \ell$ and
if $|A_t| = \ell$, then $M_{B_t,t} \geq S^*$.
\end{lemma}

\begin{proof}
$|A_t| = \max\set{j : \sum_{i=1}^{j} \ubar \nu_{\pi_t(i),t-1} \leq 1}$.
But $\ubar \nu_{k,t-1} \leq \nu_k$ for all $k$ and $t$, so
$\sum_{i=1}^\ell \ubar \nu_{\pi_t(i),t - 1} \leq \sum_{k=1}^\ell \ubar \nu_{k,t-1} \leq \sum_{k=1}^\ell \nu_k
\leq 1$.
Therefore $|A_t| \geq \ell$. If $|A_t| = \ell$, then $B_t \notin A_t$ is the overflow job and so 
$M_{B_t,t} = 1 - \sum_{k \in A_t} \ubar \nu_{k,t-1} \geq 1 - \sum_{k \in A^*} \ubar \nu_{k,t-1}
\geq 1 - \sum_{k \in A^*} \nu_k \equiv S^*$
\end{proof}
We now decompose the regret, while still assuming that $G$ holds:
\eqn{
\sum_{t=1}^n r_t  
\nonumber &= \sum_{t=1}^n \left(\ell + {S^* \over \nu_{\ell+1}} - \sum_{k=1}^K {M_{k,t} \over \nu_k} \right)\\
\label{eq:regret-d2}
&\leq \sum_{t=1}^n \sum_{k \in A^{\leq\ell}_t} \left(1 - {M_{k,t} \over \nu_k}\right) \\ 
&\qquad+\ind{\ell < K}\sum_{t=1}^n \Bigg({S^* \over \nu_{\ell+1}} - {M_{B_t,t} \over \nu_{B_t}}\Bigg).
}
Let us bound the first sum:
\eqn{
\nonumber \sum_{t=1}^n &\sum_{k \in A^{\leq\ell}_t} \left(1 - {M_{k,t} \over \nu_k}\right) \\
\nonumber &= \sum_{t=1}^n \sum_{k=1}^K \ind{k \in A^{\leq\ell}_t} \left(1 - {\ubar \nu_{k,t-1} \over \nu_k}\right) \\
\nonumber &\sr{(a)}\leq \sum_{t=1}^n \sum_{k=1}^K \ind{k \in A^{\leq\ell}_t}  \ubar\nu_{k,t-1}\epsilon_{k,t-1} \\
\label{eq:first-sum1} &\sr{(b)}\leq \sum_{t=1}^n \sum_{k=1}^K \ind{k \in A^{\leq\ell}_t} {c_{k,1} \ubar\nu_{k,t-1} \over \ubar \nu_{k,0} T_k(t)} \,,
}
where (a) follows by \cref{lem:w} and (b) by \cref{thm:fast-estimator}.
\begin{lemma}\label{lem:upper-bandit}
If $k > j$, then 
\eq{
\sum_{t=1}^n \ind{k \in A_t^{\leq j}} \leq {c_{k,1} \over \ubar \nu_{k,0} \Delta_{j,k}} \rdefined u_{j,k}.
}
\end{lemma}

\begin{proof}
Assume $k \in A_t^{\leq j}$. Therefore $\ubar \nu_{k,t-1} \leq \nu_j$. 
But if $u_{j,k} < \sum_{s=1}^{t} \ind{k \in A_s^{\leq j}} \leq T_k(t-1) + 1$, then
\eq{
{1 \over \ubar \nu_{k,t-1}} 
&\leq {1 \over \nu_k} + \epsilon_{k,t-1} 
= {1 \over \nu_j} + \epsilon_{k,t-1} - \Delta_{j,k} \\ 
&\sr{(a)}\leq {1 \over \nu_j} + {c_{k,1} \over \ubar \nu_{k,0} (T_k(t-1) + 1)} - \Delta_{j,k} 
< {1 \over \nu_j}, 
}
where (a) follows from \cref{thm:fast-estimator}.
Therefore $k \in A_t^{\leq j}$ implies that $\sum_{s=1}^t \ind{k \in A_s^{\leq j}} \leq u_{j,k}$
and so $\sum_{t=1}^n \ind{k \in A_t^{\leq j}} \leq u_{j,k}$ as required.
\end{proof}

Continuing \eqref{eq:first-sum1} by applying \cref{lem:upper-bandit} with $j = \ell$:
\eqn{
\nonumber &\sum_{t=1}^n \sum_{k=1}^K \ind{k \in A^{\leq\ell}_t} {c_{k,1} \ubar\nu_{k,t-1} \over \ubar \nu_{k,0} T_k(t)}  \\
\nonumber =~& \sum_{t=1}^n \sum_{k \in A^*} \ind{k \in A^{\leq\ell}_t} {c_{k,1} \ubar\nu_{k,t-1} \over \ubar \nu_{k,0}T_k(t)} \\ 
\label{eq:upper-bandit} 
&\qquad\quad+\sum_{t=1}^n \sum_{k \notin A^*} \ind{k \in A^{\leq\ell}_t} {c_{k,1} \ubar\nu_{k,t-1} \over \ubar \nu_{k,0} T_k(t)} \\
\nonumber \sr{(a)}\leq~& \sum_{k \in A^*} \sum_{t=1}^n {c_{k,1} \eta_k \over t} +
\sum_{k \notin A^*} \sum_{t=1}^{u_{\ell,k}} {c_{k,1} \eta_k \over t} \\
\nonumber \leq~& \sum_{k=1}^\ell c_{k,1} \eta_k(1 + \log n) + \sum_{k = \ell+1}^K c_{k,1} \eta_k(1 + \log u_{\ell,k}),
}
where (a) follows by \cref{lem:upper-bandit} and the fact that $k \in A_t^{\leq\ell}$ implies that ${\ubar \nu_{k,t-1} \over \ubar \nu_{k,0}} \leq \eta_k$.
Now if $\ell = K$, then the second term in \eqref{eq:regret-d2} is zero and the proof is completed by substituting the
above result into \eqref{eq:regret-d2} and then into \eqref{eq:regret-d1}. So now we assume $\ell > K$ and 
bound the second term in \eqref{eq:regret-d2} as follows:
\eqn{
\nonumber &\sum_{t=1}^n \left({S^* \over \nu_{\ell+1}} - {M_{B_t,t} \over \nu_{B_t}}\right)
\leq \sum_{t=1}^n \ind{B_t \in A_t} \left(1 - {\ubar \nu_{B_t,t-1} \over \nu_{B_t}}\right) \\
&\quad + \sum_{t=1}^n \ind{B_t \notin A_t} \left({S^* \over \nu_{\ell+1}} - {S^* \over \nu_{B_t}}\right),
\label{eq:lower-bandit-d}
}
where we used \cref{lem:s} and $S^* \leq 1$ and that if $B_t \in A_t$, then $M_{B_t,t} = \ubar\nu_{B_t,t-1}$.
Bounding each term separately: 
\eqn{
\nonumber &\sum_{t=1}^n \ind{B_t \in A_t} \left(1 - {\ubar \nu_{B_t,t-1} \over \nu_{B_t}}\right) \\
\nonumber \sr{(a)}\leq~& \sum_{k=1}^K \sum_{t=1}^n \ind{k \in A_t^{\leq\ell+1}} \left(1 - {\ubar \nu_{k,t-1} \over \nu_k}\right) \\
\label{eq:lower-bandit-1} 
\sr{(b)}\leq~& \sum_{k=1}^K \sum_{t=1}^n \ind{k \in A_t^{\leq\ell+1}} \ubar\nu_{k,t-1} \epsilon_{k,t-1} \\
\nonumber \sr{(c)}\leq~& \sum_{k=1}^K \sum_{t=1}^n \ind{k \in A_t^{\leq\ell+1}} {c_{k,1} \ubar\nu_{k,t-1} \over \ubar \nu_{k,0} T_{k}(t)}  \\
\nonumber \sr{(d)}\leq~&\sum_{k=1}^{\ell+1} c_{k,1}\eta_k(1 + \log n)  
+ \sum_{k=\ell+2}^K c_{k,1} \eta_k(1 + \log u_{\ell+1,k}), 
}
where (a) follows since $B_t \in A_t$ implies that $B_t \in A_t^{\leq \ell+1}$,
(b) follows from \cref{lem:w}\eqref{lem:w4}, (c) by \cref{thm:fast-estimator}, and (d) follows 
from \cref{lem:upper-bandit} and the same analysis as \eqref{eq:upper-bandit}. 
For the second term we need the following lemma, which uses \cref{thm:slow-estimator} and a reasoning analogues 
to that of \citet{ACF02} to bound the regret of the UCB algorithm for stochastic bandits:
\begin{lemma}\label{lem:lower-bandit}
Let $U_k(t) = \sum_{s=1}^t \ind{M_{k,s} \geq S^*}$ and $k > \ell+1$.
If
$U_k(t) \geq {c_{k,2} \over S^* \ubar \nu_{k,0} \Delta_{\ell+1,k}^2} \rdefined v_k$, then $k \neq B_t$.
\end{lemma}

\begin{proof}
If $\ubar \nu_{k,t-1} > \nu_{\ell+1}$, then $k \neq B_t$.
Furthermore, if $U_k(t) > v_k$, then 
\eq{
{1 \over \ubar \nu_{k,t-1}}
&\leq {1 \over \nu_k} + \epsilon_{k,t-1} 
= {1 \over \nu_{\ell+1}} - \Delta_{\ell+1,k} + \epsilon_{k,t-1} \\
&\sr{(a)}\leq {1 \over \nu_{\ell+1}} - \Delta_{\ell+1,k} + \sqrt{c_{k,2} \over \ubar \nu_{k,0} S^* U_k(t)} 
< {1 \over \nu_{\ell+1}},
}
where (a) follows from \cref{thm:slow-estimator}.
\end{proof}

Therefore
\eqn{
\nonumber &\sum_{t=1}^n \ind{B_t \notin A_t} \left({S^* \over \nu_{\ell+1}} - {S^* \over \nu_{B_t}}\right) \\
\nonumber \sr{(a)}\leq~& S^* \sum_{k=1}^{K} \sum_{t=1}^n \ind{k = B_t \notin A_t} \Delta_{\ell+1,k} \\
\nonumber \sr{(b)}\leq~& S^* \sum_{k=\ell+2}^K \sum_{t=1}^n \ind{k = B_t \notin A_t} \Delta_{\ell+1,k} \\
\nonumber \sr{(c)}\leq~& S^* \sum_{k=\ell+2}^K \sum_{t=1}^n \ind{k = B_t \wedge M_{k,t} \geq S^*} \Delta_{\ell+1,k} \\
\sr{(d)}\leq~& \sum_{k=\ell+2}^K S^* \Delta_{\ell+1,k} v_k 
\label{eq:lower-bandit-2} \sr{(e)}= \sum_{k=\ell+2}^K {c_{k,2} \over \ubar \nu_{k,0} \Delta_{\ell+1,k}},
}
where 
(a) follows from the definition of $\Delta_{\ell+1,k}$ and the fact that if $B_t \notin A_t$, then $|A_t| = \ell$,
(b) follows since $\Delta_{\ell+1,k}$ is negative for $k \leq \ell+1$,
(c) by \cref{lem:s}, 
(d) by \cref{lem:lower-bandit}, and 
(e) by the definition of $v_k$.
Substituting \eqref{eq:lower-bandit-1} and \eqref{eq:lower-bandit-2} into \eqref{eq:lower-bandit-d} we have
\eq{
&\sum_{t=1}^n \left({S^* \over \nu_{\ell+1}} - {M_{B_t,t} \over \nu_{B_t}}\right) 
\leq
\sum_{k=1}^{\ell+1}c_{k,1} \eta_k(1 + \log n)  \\ 
&+ \sum_{k=\ell+2}^K c_{k,1} \eta_k(1 + \log u_{\ell+1,k}) 
+ \sum_{k=\ell+2}^K {c_{k,2} \over \ubar \nu_{k,0} \Delta_{\ell+1,k}}.
}
We then substitute this along with \eqref{eq:upper-bandit} into \eqref{eq:regret-d2} and then \eqref{eq:regret-d1} to obtain
\eq{
&R_n \leq
1 
+ \sum_{k=1}^\ell c_{k,1} \eta_k(1 + \log n) \\ 
&+ \ind{\ell < K}\! \Bigg[
\sum_{k=\ell+2}^K {c_{k,2} \over \ubar \nu_{k,0} \Delta_{\ell+1,k}} 
+ \sum_{k=1}^{\ell+1} c_{k,1}\eta_k(1 + \log n) \\ 
& + \sum_{\mathclap{k=\ell+2}}^K c_{k,1} \eta_k(1 + \log u_{\ell+1,k}) 
+ \sum_{\mathclap{k = \ell+1}}^K c_{k,1} \eta_k(1 + \log u_{\ell,k}) 
\Bigg].
}

\section{INITIALISATION}\label{sec:init}
Previously we assumed a known lower bound $\ubar \nu_{k,0} \leq \nu_k$ for each $k$.
In this section we show that these bounds are easily obtained using a halving trick. 
In particular, the following
algorithm computes a lower bound $\ubar \nu_0 \leq \nu$ for a single job with unknown parameter
$\nu$.

\algblockdefx[ForUntil]{ForUntil}{EndForUntil}[1][]{{\bf for} #1}{{\bf until} $X_t = 0$} 
\begin{algorithm}
\caption{Initialisation of $\ubar \nu_0$}\label{alg:init}
\begin{algorithmic}[1]
\For{$t \in 1, \ldots, \infty$}
\State Allocate $M_{t} = 2^{-t}$ and observe $X_{t}$
\State {\bf if} $X_t = 0$ {\bf then return} $\ubar \nu_0 \leftarrow 2^{-t}$.
\EndFor
\end{algorithmic}
\end{algorithm}

A naive way to eliminate the need for the lower bounds $(\ubar \nu_{k,0})_k$ is simply to run
Algorithm \ref{alg:init} for each job sequentially. Then the following proposition 
shows
that $\eta \in O(1)$ is reasonable, which justifies the claim made in \eqref{eq:big-o-theorem} that the $\eta_k$ terms
appearing in \cref{thm:main} are $O(1)$.

\begin{proposition}
If $\eta = {\min\set{1, \nu} \over \ubar \nu_0}$, then $\E\eta \leq 4$.
\end{proposition}

\ifsup
\begin{proof}
Let $p_t$ be the probability that the algorithm ends after time-step $t$, which is
\eq{
p_t = (1 - \beta(2^{-t} / \nu)) \prod_{s=1}^{t-1} \beta(2^{-s} / \nu).
}
Therefore
\eq{
\E\eta &= \E\left[\min\set{1,\nu} \over \ubar \nu_0\right] 
= \sum_{t=1}^\infty p_t \cdot {\min\set{1, \nu} \over M_t} \\
&= \min\set{1, \nu}\sum_{t=1}^\infty 2^t (1 - \beta(2^{-t} / \nu)) \prod_{s=1}^{t-1} \beta(2^{-s} / \nu) \\
&\leq 4,
}
where the final inequality follows from an arduous, but straight-forward, computation.
\end{proof}
\else
\fi
The problem with the naive method is that 
the expected running time of Algorithm \ref{alg:init} is $O(\log{1 \over \nu})$, which may be arbitrary large for small $\nu$ and lead to 
a high regret \emph{during the initialisation period}. Fortunately, the situation 
when $\nu$ is small is easy to handle, since the amount of resources required to complete such a job is also small.
The trick is to run $K$ offset instances of Algorithm \ref{alg:init} alongside a modified version of Algorithm \ref{alg:OAA}. First we describe the 
parallel implementations of
Algorithm \ref{alg:init}. For job $k$, start Algorithm \ref{alg:init} in time-step $k$, which means that the total amount of resources
used by the parallel copies of Algorithm \ref{alg:init} in time-step $t$ is bounded by 

\noindent
\hspace{-0.3cm}
\begin{minipage}{3.8cm}
\vspace{-0.5cm}
\eqn{
\nonumber &\sum_{k=1}^K \ind{t \geq k} 2^{k-t-1} \\
\label{eq:parallel} &\leq \min\set{1, 2^{K-t}}. 
}
\end{minipage}
\hspace{0.3cm}
\begin{minipage}{4cm}
\vspace{-0.4cm}
\scriptsize
\renewcommand{\arraystretch}{1.2}
\begin{tabular}{|p{0.5cm}p{0.35cm}p{0.35cm}p{0.35cm}p{0.4cm}|}
\hline 
Job   & $M_{k,1}$     & $M_{k,2}$     & $M_{k,3}$     & $M_{k,4}$        \\
$1$ & $\nicefrac{1}{2}$ & $\nicefrac{1}{4}$ & $\nicefrac{1}{8}$ & $\nicefrac{1}{16}$ \\
$2$ & $0$           & $\nicefrac{1}{2}$ & $\nicefrac{1}{4}$ & $\nicefrac{1}{8}$ \\
$3$ & $0$           & $0$           & $\nicefrac{1}{2}$ & $\nicefrac{1}{4}$ \\ \hline
$\!\!\!\!\sum\limits_{k=1}^K\! M_{k,t}$& $\nicefrac{1}{2}$ & $\nicefrac{3}{4}$ & $\nicefrac{7}{8}$ & $\nicefrac{7}{16}$ \\ \hline
\end{tabular} 
\end{minipage}

Algorithm \ref{alg:OAA} is implemented starting from time-step $1$, but only allocates resources to jobs for which the initialisation process has completed. Estimates
are computed using only the samples for which Algorithm \ref{alg:OAA} chose the allocation, which 
ensures that they are based on allocations with $M_{k,t} \leq \nu_k$.
Note that the analysis of the modified algorithm does not depend on the order in which the parallel processes are initialised.
The regret incurred by the modified algorithm is given in order notation in \eqref{eq:big-o-theorem}.
The proof is omitted, but relies on two observations. First, that the expected number of time-steps that a job is not (at least) fully allocated while
it is being initialised is $2$. The second is that the resources available to Algorithm \ref{alg:OAA} at time-step $t$ converges exponentially
fast to $1$ by \eqref{eq:parallel}.


\section{MINIMAX LOWER BOUNDS}
\vspace{-0.3cm}
Despite the continuous action space, the techniques used when proving minimax lower bounds for standard stochastic bandits \citep{ACFS95} can be 
adapted to our setting.

\begin{theorem}
Given fixed $n$ and $8n \geq K \geq 2$ and an arbitrary algorithm, there exists an allocation problem 
for which the expected regret satisfies $R_n \geq {\sqrt{nK} \over 16 \sqrt{2}}$.
\end{theorem}

\newcommand{\Pk}{\mathbb P}

\ifsup
\begin{proof}
Let $1 \geq \epsilon > 0$ be a constant to be chosen later.
We consider a set of $K$ allocation problems where in problem $k$, 
$\nu_j = 2$ for all $j \neq k$ and $\nu_k = {2 \over 1 + \epsilon}$.
The optimal action in problem $k$
is to assign all available resources to job $k$ when the expected reward is ${1 + \epsilon \over 2}$.
The interaction between the algorithm and a problem $k$ defines a measure $\Pk_k$ on the set of outcomes (successes, allocations). We write
$\E_k$ for expectations with respect to measure $\Pk_k$. We have
\eqn{
\label{eq:lower-d} \E_k\left[\sum_{t=1}^n M_{k,t}\right] - \E_0 \left[\sum_{t=1}^n M_{k,t}\right] 
&\leq n\sqrt{{1 \over 2} \KL(\Pk_0, \Pk_k)}\,,
}
where $\KL(\Pk_0,\Pk_k)$ is the Kullback-Leibler divergence (or relative entropy) between $\Pk_0$ and $\Pk_k$.
The divergence is bounded by
\eqn{
\nonumber&\KL(\Pk_0, \Pk_k) 
\sr{(a)}\leq \E_0 \left[\sum_{t=1}^n {\epsilon^2 M_{k,t}^2 \over 4} \left({1 \over {M_{k,t} \over 2}} + {1 \over 1 - {M_{k,t} \over 2}}\right)\right] \\
&\sr{(b)}= \E_0 \left[\sum_{t=1}^n {\epsilon^2 M_{k,t} \over 2 - M_{k,t}} \right]
\label{eq:KL} \sr{(c)}\leq \epsilon^2 \E_0 \left[\sum_{t=1}^n M_{k,t}\right]\,,
}
where (a) follows from the telescoping property of the KL divergence and by bounding the KL divergence by the $\chi$-squared distance,
(b) is trivial and (c) follows since $M_{k,t} \leq 1$.
The $n$-step expected regret given environment $k$ is
\eqn{
\nonumber R_n(k) 
&= {n(1 + \epsilon) \over 2} - \E_k \sum_{t=1}^n \sum_{j=1}^K {M_{j,t} \over \nu_j} \\
\label{eq:lower-regret} &\sr{(b)}\geq {\epsilon \over 2} \left(n - \E_k \sum_{t=1}^n M_{k,t} \right) 
}
where (b) follows by recalling that $\nu_j = 2$ unless $j = k$, when $\nu_j = 2/(1+\epsilon)$.
Therefore,
\eq{
&\sup_k R_n(k) 
\sr{(a)}\geq {1 \over K} \sum_{k=1}^K R_n(k) \\
&\sr{(b)}\geq {1 \over K} \sum_{k=1}^K {\epsilon \over 2} \left(n - \E_k \sum_{t=1}^n M_{k,t} \right) \\ 
&\sr{(c)}\geq {1 \over K} \sum_{k=1}^K {\epsilon \over 2} \left(n - \E_0 \sum_{t=1}^n M_{k,t} - n\epsilon \sqrt{{1 \over 2} \E_0 \sum_{t=1}^n M_{k,t}}\right)  \\
&\sr{(d)}\geq {\epsilon \over 2K} \left(nK - n - n\epsilon \sum_{k=1}^K \sqrt{{1 \over 2}\E_0 \sum_{t=1}^n M_{k,t}}\right) \\
&\sr{(e)}\geq {\epsilon \over 2K} \left(nK - n - n\epsilon \sqrt{{K \over 2} \sum_{k=1}^K \E_0 \sum_{t=1}^n M_{k,t}}\right) \\
&\sr{(f)}\geq {\epsilon \over 2K} \left(nK - n - n\epsilon \sqrt{nK \over 2}\right) 
\sr{(g)}\geq {\epsilon n \over 4} - {\epsilon^2 n^{3 \over 2} \over 2\sqrt{2} K^{1 \over 2}}\,,
}
where (a) follows since the max is greater than the average,
(b) follows from \eqref{eq:lower-regret},
(c) is obtained by combining \eqref{eq:lower-d} and \eqref{eq:KL},
(d) follows from the fact that $\sum_{k=1}^K M_{k,t} \leq 1$,
(e) is true by Jensen's inequality and (f/g) are straight-forward.
Choosing $\epsilon = \sqrt{K / (8 n)}$ leads to $\sup_k R_n(k) \geq {\sqrt{nK} /( 16 \sqrt{2})}$ as required.
\end{proof}
\else
\fi

\section{EXPERIMENTS}
\vspace{-0.3cm}
Data points were generated using the modified algorithm described in Section \ref{sec:init} and by taking the mean of 300 samples.
With this many samples the standard error is relatively low (and omitted for readability). We should note that the variance in the regret of the modified algorithm
is reasonably large, because the regret depends linearly on the random $\eta_k$. For known lower bounds the variance is extremely low.
To illustrate the behaviour of the algorithm we performed four experiments on synthetic data with $K = 2$, which are plotted below as TL (top left), TR, BL, BR (bottom right) respectively.
In TL we fixed $n = 10^4$, $\nu_1 = 2$ and plotted
the regret as a function of $\nu_2 \in [2, 10]$. The experiment shows the usual bandit-like dependence on the gap $1/\Delta_{1,2}$. 
In TR we fixed $\nu_1 = \nicefrac{4}{10}$, $\nu_2 = \nicefrac{6}{10}$ and plotted $R_n / \log^2 n$ as
a function of $n$. The experiment lies within case 2 described in \cref{sec:theorems} and shows that the algorithm suffers regret
$R_n \approx 45 \log^2 n$ as predicted by \cref{thm:main}.
In BL we fixed $n = 10^5$, $\nu_1 = \nicefrac{4}{10}$ and plotted the regret as a function of $\nu_2 \in [\nicefrac{4}{10},  1]$. The results show the algorithm suffering $O(\log^2n)$
regret for both processes until the critical point when $\nu_2 > \nicefrac{6}{10}$ when the 
second process can no longer be fully allocated, which is quickly learned and
the algorithm suffers $O(\log^2 n)$ regret for only one process.
In BR we fixed $\nu_1 = \nicefrac{4}{10}$ and $\nu_2 = \nicefrac{6}{10}$ and plotted the regret as a function of $n$ for two algorithms.
The first algorithm (solid blue) is the modified version of Algorithm
\ref{alg:OAA} as described in \cref{sec:init}. The second (dotted red) 
is the same, but uses the unweighted estimator $w_{k,t} = 1$ for all $k$ and $t$.
The result shows that both algorithms suffer sub-linear regret, but that the weighted estimator is a significant improvement over the unweighted one.

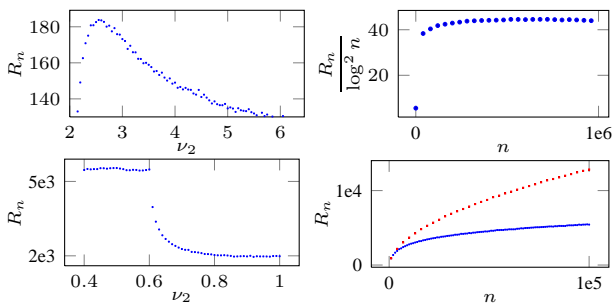
\todok{(1) TR plot is empty and the x-label is not clear. (2) Perhaps,
  make the y-ticks 5e3 and 10e3 and not 5,000 and 10,000. We have too
  much space for the y-axis and little for the actual plot. (3) colour plot
  is not clear when printed in BW.  }
\begin{tikzpicture}[baseline,font=\scriptsize]
  \begin{axis}[xlabel shift=-5pt,ylabel shift=-5pt,xlabel={$\nu_2$},ylabel={$R_n$},xmin=2,ymin=130,height=3cm,width=4.7cm,mark size=0.2pt,compat=newest]
    \addplot+[only marks] table {exp1.txt};
  \end{axis}
\end{tikzpicture}
\begin{tikzpicture}[baseline,font=\scriptsize]
  \begin{axis}[xlabel shift=-5pt,ylabel shift=-5pt,xlabel={$n$},ylabel={$\displaystyle {R_n \over \log^2 n}$},compat=newest,mark size=0.7pt,height=3cm,scaled ticks=false,width=4.38cm,xtick={0,1000000},xticklabels={0,1e6}] 
    \addplot+[only marks] table {exp2.txt};
  \end{axis}
\end{tikzpicture}

\vspace{-0.3cm}
\begin{tikzpicture}[font=\scriptsize]
  \begin{axis}[ytick={2000,5000},yticklabels={2e3,5e3},xlabel shift=-5pt,ylabel shift=-5pt,xlabel={$\nu_2$},ylabel={$R_n$},compat=newest,mark size=0.2pt,height=3cm,width=4.7cm] 
    \addplot+[only marks] table {exp3.txt};
  \end{axis}
\end{tikzpicture}
\begin{tikzpicture}[font=\scriptsize]
  \begin{axis}[ytick={0,10000},yticklabels={0,1e4},xlabel shift=-5pt,ylabel shift=-5pt,scaled ticks=false,xlabel={$n$},ylabel={$R_n$},compat=newest,mark size=0.2pt,height=3cm,width=4.74cm,xtick={0,100000},xticklabels={0,1e5}] 
    \addplot+[only marks,blue] table[x index=0,y index=1] {exp4.txt};
    \addplot+[each nth point=3, filter discard warning=false, unbounded coords=discard,only marks,red] table[x index=0,y index=2] {exp4.txt};
  \end{axis}
\end{tikzpicture}

\section{CONCLUSIONS}
\vspace{-0.3cm}
We introduced the linear stochastic resource allocation problem and a
new optimistic algorithm for this setting. Our main result shows that
the new algorithm enjoys a (squared) logarithmic problem-dependent
regret.  We also presented a minimax lower bound of
$\Omega(\sqrt{nK})$, which is consistent with the problem-dependent
upper bound. The simulations confirm the theory and highlight the
practical behaviour of the new algorithm. There are many open
questions and possibilities for future research.  Most important is
whether the $\log^2n$ can be reduced to $\log
n$. Problem-dependent lower bounds would be interesting. 
The algorithm is not anytime (although a doubling trick presumably
works in theory). Developing and analysing algorithms when the
horizon it not known, and have high-probability bounds are both of
interest. We also wonder if Thompson sampling can be efficiently implemented for
some reasonable prior, and if it enjoys the same practical and
theoretical guarantees in this domain as it does for bandits. Other
interesting extensions are when resources are not replenished, or the
state of the jobs follow a Markov process. Finally, we want to emphasise that
we have made just the first steps towards developing this new and interesting setting. We hope
to see significant activity extending and modifying the model/algorithm for specific problems.

\vspace{-0.5cm}
\paragraph{Acknowledgements}
This work was supported by the Alberta Innovates Technology Futures,
NSERC, by EU Framework 7 Project No.
248828 (ADVANCE), and by Israeli Science Foundation grant ISF-
1567/10. Part of this work was done while Csaba Szepesv\'ari was visiting
Technion. 

\appendix

\ifsup
\section{TECHNICAL INEQUALITIES}

The proof of the following theorem is given in the supplementary material.

\begin{theorem}\label{thm:peeled-bernstein}
Let $\delta \in (0,1)$ and $X_1, \ldots, X_n$ be a sequence of random variables adapted to filtration $\set{\calF_t}$ with
$\E[X_t|\calF_{t-1}] = 0$.
Let $R_t$ be $\calF_{t-1}$-measurable such that $|X_t|\le R_t$ almost surely,
$R = \max_{t \leq n} R_t$.
Define $S = \sum_{t=1}^n X_t$, $V^2 = \sum_{t=1}^n \Var[X_t|\calF_{t-1}]$,
and 
\eq{
\delta_{r,v} &= {\delta \over 3(r+1)^2 (v+1)^2}\,, \\
f(r, v) &= {r+1 \over 3} \log{2 \over \delta_{r,v}}  \\
&\quad + \sqrt{2 (v+1) \log{2 \over \delta_{r,v}} + \left({r+1 \over 3}\right)^2 \log^2{2 \over \delta_{r,v}}}\,.
}
Then
$\displaystyle \P{|S| \geq f(R, V^2)} \leq \delta$.
\end{theorem}

\begin{proof}[Proof of \cref{thm:peeled-bernstein}]
Note that $f(r, v)$ is strictly monotone increasing in both $r$ and $v$.
We now use a peeling argument. We have,
\eq{
&\P{|S_n| \geq f(R, V^2)} \\
&\sr{(a)}= \sum_{r=1}^\infty \sum_{v=1}^\infty \P{|S_n| \geq f(R, V^2), \ceil{V^2} = v, \ceil{R} = r} \\
&\sr{(b)}\leq \sum_{r=1}^\infty \sum_{v=1}^\infty \P{|S_n| \geq f(r-1, v-1), \ceil{V^2} = v, \ceil{R} = r} \\
&\sr{(c)}\leq \sum_{r=1}^{\infty} \sum_{v=1}^{\infty} 2 \exp\left(-{f(r-1, v-1)^2 \over 2 v + {2r f(r-1, v-1) \over 3}}\right) \\\
&\sr{(d)}\leq \sum_{r=1}^{\infty} \sum_{v=1}^{\infty} \delta_{r-1,v-1} 
\sr{(e)}= {\delta \over 3} \sum_{r=1}^\infty \sum_{v=1}^{\infty} {1 \over v^2 r^2} 
\sr{(f)}\leq \delta \,,
}
where (a) follows from the positivity of $R$ and $V$,
(b) by the monotonicity of $f$,
(c) by  \cref{thm:bernstein} stated below (a martingale version Bernstein's inequality),
(d) by \cref{lem:bernstein-invert},
(e) by the definition of $\delta_{r,v}$,
(f) is trivial.
\end{proof}

\begin{theorem}[Theorem 3.15 of \citealt{McD98}, see also \citealt{Fre75} and \citealt{Ber46}]\label{thm:bernstein}
Let $X_1, \ldots, X_n$ be a sequence of random variables adapted to the filtration $\set{\calF_t}$ with
$\E[X_t|\calF_{t-1}] = 0$.
Further, let $R_t$ be $\calF_{t-1}$-measurable such that $X_t\le R_t$ almost surely,
$R = \max_{t \leq n} R_t$ and 
$V^2 = \sum_{t=1}^n \Var[X_t|\calF_{t-1}]$.
Then for any $\epsilon, r, v > 0$, 
\eq{
\P{\sum_{t=1}^n X_t \geq \epsilon, V^2 \leq v, R \leq r} \leq  \exp\left(- {\epsilon^2 \over 2 v + {2 \epsilon r \over 3}} \right).
}
\end{theorem}
We note that although this inequality is usually stated for deterministic $R_t$, the extension is trivial:
Just define $Y_t = X_t \,\ind{R_t\le r}$ and apply the standard inequality to $Y_t$. The result then follows since on $R\le r$, 
$Y_t=X_t$ for all $t$ and thus $\sum_{t=1}^n X_t  = \sum_{t=1}^n Y_t$.

\begin{lemma}\label{lem:bernstein-invert}
If $\epsilon \geq {r \over 3} \log{2 \over \delta} + \sqrt{2 v \log{2 \over \delta} + {r^2 \over 9} \log^2 {2 \over \delta}}$, then
$2\exp\left(-{\epsilon^2 \over 2v + {2\epsilon r \over 3}}\right) \leq \delta$.
\end{lemma}
\fi

\ifsup
\section{TABLE OF NOTATION}

\renewcommand{\arraystretch}{1.2}
\noindent
\begin{tabular}{@{}p{1.2cm}p{6.6cm}}
$K$                   & number of jobs \\
$n$                   & time horizon \\
$\nu_k$               & parameter characterising difficulty of job $k$ \\
$\beta(p)$            & function $\beta(p) \defined \min\set{1, p}$ \\
$M_{k,t}$             & resources assigned to job $k$ in time-step $t$ \\
$X_{k,t}$             & outcome of job $k$ in time-step $t$ \\
$\ubar \nu_{k,t}$     & lower bound on $\nu_k$ at time-step $t$ \\
$\bar \nu_{k,t}$      & upper bound on $\nu_k$ at time-step $t$ \\
$\delta$              & bound on probability that some confidence intervals fails \\
$\pi_t(i)$            & $i$th easiest job at time-step $t$ sorted by $\ubar \nu_{k,t-1}$ \\
$\ell$                & number of fully allocated jobs under optimal allocation \\
$S^*$                 & optimal amount of resources assigned to overflow process \\
$A^*$                 & contains the $\ell$ easiest jobs (sorted by $\nu_k$) \\
$A_t$                 & set of jobs with $M_{k,t} = \ubar \nu_{k,t-1}$ at time-step $t$ \\
$B_t$                 & equal to $\pi_t(\ell+1)$ \\
$\eta_k$              & ${\min\set{1, \nu_k} \over \ubar \nu_{k,0}}$ \\
\end{tabular}
\fi

\bibliographystyle{plainnat}
\bibliography{mem-bandits}

\end{document}